\documentclass{article}

\usepackage[preprint, nonatbib]{neurips_2023}

\usepackage[utf8]{inputenc} 
\usepackage[T1]{fontenc}    
\usepackage{hyperref}       
\usepackage{url}            
\usepackage{booktabs}       
\usepackage{amsfonts}       
\usepackage{nicefrac}       
\usepackage{microtype}      
\usepackage[dvipsnames]{xcolor}         

\usepackage{amsmath}
\usepackage{amssymb}
\usepackage{amsthm}
\usepackage{mathtools}
\usepackage{cleveref}
\usepackage{subfigure}

\newtheorem{theorem}{Theorem}[section]

\newtheorem{conjecture}[theorem]{Conjecture}

\theoremstyle{definition}
\newtheorem{definition}{Definition}[section]

\theoremstyle{remark}

\def\*#1{\boldsymbol{\mathbf{#1}}}

\DeclareMathOperator*{\argmin}{arg\,min}

\newcommand{\Trp}{{\mathsf{T}}}

\newcommand{\DKL}[2]{\operatorname{D_{\text{KL}}}\!\left(#1\,\middle\| \,#2\right)}

\usepackage{tikz}
\usetikzlibrary{calc}
\usetikzlibrary{math}

\usepackage{pgfplots}
\pgfplotsset{compat=1.17} 
\usepackage{pgfplotstable}

\usepackage{cbcolors}

\newcommand{\norm}[1]{\left\lVert#1\right\rVert}

\title{Small-data Reduced Order Modeling of Chaotic Dynamics through SyCo-AE: Synthetically Constrained Autoencoders}

\author{{Andrey A. Popov}\\
	Oden Institute for Computational \\
Engineering \& Sciences\\
	The University of Texas at Austin\\
	Austin, TX 78712 \\
	\texttt{andrey.a.popov@utexas.edu} \\
	\And
	{Renato Zanetti} \\
        Dept. of Aerospace Engineering and \\
        Engineering Mechanics\\
	The University of Texas at Austin\\
	Austin, TX 78712 \\
	\texttt{renato@utexas.edu} \\
}

\begin{document}

\maketitle

\begin{abstract}
Data-driven reduced order modeling of chaotic dynamics can result in systems that either dissipate or diverge catastrophically.
    Leveraging non-linear dimensionality reduction of autoencoders and the freedom of non-linear operator inference with neural-networks, we aim to solve this problem by imposing a synthetic constraint in the reduced order space.
    The synthetic constraint allows our reduced order model both the freedom to remain fully non-linear and highly unstable while preventing divergence.
    We illustrate the methodology with the classical 40-variable Lorenz '96 equations, showing that our methodology is capable of producing medium-to-long range forecasts with lower error using less data.
\end{abstract}

\section{Introduction}

The use of machine learning methods in many scientific disciplines is hampered by the availability of data~\cite{kitchin2015small}, resulting in the small-data problem.
Knowledge-guided~\cite{karpatne2022knowledge} machine learning~\cite{aggarwal2018neural,goodfellow2016deep} aims to, in part, solve this problem by augmenting data driven methods with \textit{a priori} knowledge about the true system behavior. 
This prior knowledge can act as a regularizer, which, for instance, can restricting neural network outputs to remain physically consistent.

Chaotic systems---commonplace in numerical weather prediction and data assimilation applications~\cite{kalnay2003atmospheric,reich2015probabilistic,asch2016data}---have the peculiar property of both magnifying small perturbations and damping them at the same time~\cite{strogatz2018nonlinear,guckenheimer2013nonlinear}. Because of this, the dynamics of chaotic systems are hard to predict, being linearly unstable on average, while the limit set of the dynamics, the attractor, remains compact~\cite{farmer1983dimension}, resulting in quasi-periodic behavior.
In practice all the points of the attractor cannot be known: it is possible that there is very limited data about the attractor both temporally and spatially. 
These attractors often possess complicated topological properties~\cite{rand1978topological} and are difficult to work with for high dimensional systems,.

Reduced order modeling (ROM)~\cite{benner2017model} combines the ideas of dimensionality reduction and the idea of operator inference~\cite{brunton2022data} in order to build efficient models of dynamical systems.
In the world on machine learning dimensionality reduction is frequently performed through the use of autoencoders~\cite{weng2018VAE,popov2022meta,goodfellow2016deep}, which are the focus of this work.

In this work, we make the simple assumption that we can construct an autoencoder and reduced order dynamics such that the chaotic attractor can be embedded into a simple-to-define set in reduced order space. 
This can be achieved by imposing a \textit{synthetic constraint} on both the autoencoder and the reduced order dynamics.
If the synthetic constraint defines a compact set, we can guarantee that the reduced order dynamics stay bounded while simultaneously having all the non-linear freedom that generalized function approximators allow.
The synthetic constraint therefore becomes a \textit{simple and known} transformation of some 
\textit{unknown} constraint of the full model, and, acts as a proxy for our knowledge about the behavior of chaotic systems.

This work aims to build reduced order models of chaotic systems by combining autoencoders,  fully non-linear operator inference, and synthetic constraints.
We call this framework synthetically constrained autoencoders, shortened to SyCo-AE. 
We provide a theoretical justification for the derivation of the SyCo-AE through the construction of the strong and weak preservation properties.
We show how an ideal reduced order model satisfies the strong preservation property, and how the SyCo-AE is built explicitly to preserve the weak preservation property.
We additionally provide a novel numerical method for training this reduced order model, by embedding the solution of a constrained differential equation into the cost function.

We apply the SyCo-AE framework to the Lorenz '96 equations~\cite{lorenz1996predictability}, which is a common medium-scale highly-chaotic test problem. 
Our results show, both qualitatively and quantitatively, that the SyCo-AE framework is capable of producing a reduced order model that can produce reliable medium-range forecasts of chaotic dynamics.


\section{Background}

Reduced order modeling combines two concepts:
\textit{dimensionality reduction}, and \textit{operator inference}. The former primarily deals with preserving \textit{spatial} features of the dynamical system through a compressed representation while the latter focuses on preserving the \textit{temporal} features of the dynamical system.

We now describe dimensionality reduction from the point of view of autoencoders.
The autoencoder,
\begin{equation}\label{eq:autoencoder}
    \begin{aligned}
        u &= \theta(x),\\
        x&\approx \widetilde{x} = \phi(u),
    \end{aligned}
\end{equation}
aims to take information given by the high-dimensional state $x$, and distill it down to some useful lower-dimensional representation $u$, then back out again into a reconstruction $\widetilde{x}$.
More formally, assume that $x$ is an element of the full order space $\mathbb{X}$ which is a subset of $n$-dimensional Euclidean space $\mathbb{X}\subset \mathbb{R}^n$.     
The encoder is a function $\theta: \mathbb{R}^n \xrightarrow[]{} \mathbb{R}^r$, with the reduced order space defined by the image of the full order space under the encoder, $\theta(\mathbb{X}) = \mathbb{U}$,  which itself is a subset of $r$-dimensional Euclidean space,
$\mathbb{U} \subset \mathbb{R}^r$. 
The decoder is a function $\phi: \mathbb{R}^r \xrightarrow[]{} \mathbb{R}^n$, with the reconstruction space defined by the image of the reduced order space under the decoder $\phi(\mathbb{U}) = \widetilde{\mathbb{X}}$.

The encoder and decoder in \cref{eq:autoencoder} are functions that are required to have additional properties for dimensionality reduction to be valid. A necessary, but not sufficient condition for valid dimensionality reduction on the encoder-decoder pair is that of right-invertibility~\cite{popov2022meta},
\begin{equation}\label{eq:right-invertibility}
    \theta(\phi(\theta(x))) = \theta(x), \quad \forall x\in\mathbb{X},
\end{equation}
over all the data. This property makes sure that there is no loss of information from the reduced order representation  $u$ in the reconstruction $\widetilde{x}$, as it would encode into the exact same $u$.

The true state that we are trying to model is assumed to come from some continuous dynamical system defined by the differential equation,
\begin{equation}\label{eq:full-order-ODE}
\begin{gathered}
        \frac{\mathrm{d} x}{\mathrm{d} t} = F(x),
\end{gathered}
\end{equation}
with $F$ representing the (possibly highly non-linear) dynamics. The goal of reduced order modeling is to find a way to approximate the \textit{full order model}~\cref{eq:full-order-ODE} in the reduced space defined by the autoencoder~\cref{eq:autoencoder}.

In \textit{intrusive} reduced order modeling, the full order model~\cref{eq:full-order-ODE} is known explicitly and is used in order to take advantage of all available information.
However when~\cref{eq:full-order-ODE} is not known, or when our knowledge about it is severely deficient, intrusive methods cannot be relied upon. In this work we focus on \textit{non-intrusive} methods that do not have access to the full order model, and must rely only on data.

Given the state of the dynamics at time index $i$ in the reduced space, denoted by $u_i$, the goal is to find the state of the dynamics at time index $i+1$, denoted by $u_{i+1}$. Conceptually $u_{i+1}$ can be approximated in two major ways. The first is that of \textit{flow maps}~\cite{qin2019data}, where a simple function that maps one to the other is learned,
\begin{equation}\label{eq:flow-map}
    u_{i+1} \approx \mathcal{F}(u_i).
\end{equation}
While conceptually simple, this approach does not lend well to continuous dynamical systems, as the attractors for discrete and continuous dynamics do not behave in the same manner~\cite{guckenheimer2013nonlinear}.
The alternate approach, which we utilize in this paper is that of \textit{operator inference}, whereby the action of the continuous time dynamics in the reduced space is learned explicitly, 
\begin{equation}\label{eq:reduced-order-IVP}
        \frac{\mathrm{d} u}{\mathrm{d} t} = f(u),\quad u(t_i) = u_i,\quad t\in[t_i, t_{i+1}],
\end{equation}
where $f$ is some (potentially highly non-linear) function defining the dynamics in the reduced space, $t$ is the time of the full order dynamics, and $u_i$ is the initial condition. The solution of the initial value problem~\cref{eq:reduced-order-IVP} can be performed with a wide array of algorithms~\cite{hairer1991solving}.


The most straight-forward autoencoder-based approach to reduced order modeling of continuous dynamics take a fully non-linear autoencoder-based dimensionality reduction method coupled to a fully neural-network-based operator inference step,
\begin{equation}\label{eq:autoencoder-ROM}
    \begin{gathered}
        u = \theta(x),\quad \widetilde{x} = \phi(u),\\
        \frac{\mathrm{d} u}{\mathrm{d} t} = f(u),
    \end{gathered}
\end{equation}
where $\theta$ and $\phi$ are the autoencoder~\cref{eq:autoencoder}, and $f$ is a full non-linear approximation of the $f$ found in~\cref{eq:reduced-order-IVP}.
In this work we augment the approach in~\cref{eq:autoencoder-ROM}.


\subsection{Other non-intrusive methods}
\label{sec:previous-methods}

We now discuss some previous non-intrusive methods, that this work compares against.

One method for performing non-intrusive reduced order modeling is dynamic mode decomposition (DMD)~\cite{brunton2022data,kutz2016dynamic}.
Given a set of data points, $\{(x_i, x_{i+1})\}_i$, DMD finds the best rank-$r$ linear operator $\mathbf{A}$ that transports the data from time index $i$ to time index $i+1$, given by  $x_{i+1} \approx \mathbf{A}x_i$. 
Taking the  eigendecomposition of $\mathbf{A}$, given by
$\*A\*\Phi = \*\Phi\*\Lambda$, we can define the reduced order model in the following way,
\begin{equation}\label{eq:DMD-ROM}
\begin{gathered}
    u = \theta(x) = \*\Phi^\dagger x,\quad \widetilde{x} = \phi(u) = \*\Phi u,\\
    \frac{\mathrm{d} u}{\mathrm{d} t} = \*\Omega u,\quad \*\Omega = \frac{1}{\Delta t}\log\*\Lambda,
\end{gathered}
\end{equation}
where $\*\Phi$ is our linear decoder, its pseudo-inverse, $\*\Phi^\dagger$ is the encoder, and $\*\Omega$ is the linear operator defining the time dynamics. The initial value problem~\cref{eq:DMD-ROM} has a linear analytic solution, which is used in this work.

The advantage of linear methods is that they require few data to converge to a useful solution. In general, the amount of data required to use a linear method is on the order of the reduced dimension $r$ and significantly less than the dimension of the full order model $n$.

However linear methods fail to produce useful results for highly non-linear systems and are incapable of modeling chaotic dynamics. As such, non-linear  reduced order modeling methods are required.

Moving away from linear methods, a recent more-than-linear approach is that of quadratic manifolds~\cite{geelen2023operator}, which requires $\mathcal{O}(r^2)$ data points to construct,
\begin{equation}\label{eq:quadratic-rom}
\begin{gathered}
    u = \theta(x) = \*\Phi^\Trp(x - \bar{x}),\quad \widetilde{x} = \phi(u) = \bar{x} + \*\Phi u + u^\Trp\, \overline{\*\Phi} u,\\
    \frac{\mathrm{d} u}{\mathrm{d} t} = a + \*B u + u^\Trp \mathcal{C} u,
    \end{gathered}
\end{equation}
where $\bar{x}$ can be taken to be the mean-field of the data, $\*\Phi$ is a matrix constructed by proper orthogonal decomposition, and $\overline{\*\Phi}$ is a 3-tensor defining the quadratic correction term. The terms $a$, $\*B$ and $\mathcal{C}$ represent the quadratic approximation of the reduced order model.
The quadratic manifolds approach first constructs the dimensionality reduction, and then constructs the dynamics from the data approximating the time derivative of $u$ through finite difference.
Sparsity in the linear operators and 3-tensors can be enforced in a way similar to the sparse identification of non-linear dynamics (SINDy) method~\cite{brunton2016discovering}, though this type of regularization is not explored in this work.

In both the classic approaches above, the encoder is an affine transformation, meaning that the image of $\mathbb{R}^n$ under the encoder is $\theta(\mathbb{R}^n) = \mathbb{R}^r$, and not a compact superset of the reduced order space $\mathbb{U}$. Thus, the encoder is not restricted to producing results in $\mathbb{U}$. This means that they can produce results that are not meaningful to the problem at hand.

\section{Synthetically Constrained Autoencoder}

\begin{figure}
    \centering
    \includegraphics[width=\linewidth]{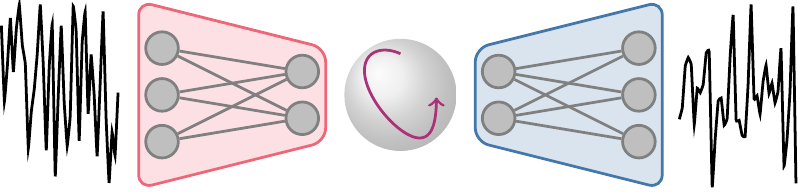}
    \caption{A visual representation of the SyCo-AE framework. From left to right: full order model~\cref{eq:full-order-ODE} data is encoded into the reduced order space, the reduced order model~\cref{eq:reduced-order-IVP} is evolved on the sphere $S^{r-1}$, and then decoded out by the decoder, into the reconstruction. Both left and right curves in this figure are real data from the Lorenz '96 model described in~\cref{sec:numerical-experiments}.}
    \label{fig:SyCo-AE-Visual}
\end{figure}

If we know that the system that we wish to model is chaotic, then, given an encoder that preserves compactness, the reduced order dynamics have to always remain on some compact set.
If they do not then our constructed reduced order model did not take advantage of all our available knowledge, which is a violation of our current best understanding of scientific reasoning~\cite{jaynes2003probability}.
This type of violation is often done deliberately for lack of a better solution.
Our task, therefore, is to somehow include this knowledge into our reduced order model construction.

A straightforward consequence of this reasoning is that the construction of the autoencoder~\cref{eq:autoencoder} cannot be performed independently of the construction of the reduced order dynamics~\cref{eq:reduced-order-IVP}.
We have to not only ensure that the reduced order dynamics produce outputs that are in accordance with the reduced order space induced by the encoder, but also ensure the complement: that the encoder produces outputs that are in accordance to the outputs produced by the reduced order dynamics.

We now formalize the discussion above by defining a few properties, and their straightforward consequences.

\begin{definition}\label{def:strong-preservation-condition}
Ideally the dynamics in~\cref{eq:reduced-order-IVP} evolve in a way such that the final state is always in the reduced order space $u_{i+1}\in\mathbb{U}$, which we term this the \textit{strong preservation condition}.
\end{definition}

We now outline how the strong preservation condition in~\cref{def:strong-preservation-condition} could be satisfied when the attractor is known.

\begin{theorem}\label{prop:strong-condition}
When all the possible data points of the attractor $\mathbb{X}$ are fully known, and $\theta$, $\phi$ and $f$ in the standard autoencoder reduced order model~\cref{eq:autoencoder-ROM} are allowed to have an arbitrary amount of degrees of freedom, then the strong preservation condition in \cref{def:strong-preservation-condition} is preservable.
\end{theorem}
\begin{proof}
    With arbitrary many degrees of freedom, the discrepancy of the propagation $u_{i+1}$ and the dimensionality reduced representation of the truth, $\theta(x_{i+1})$, can be zero almost surely for all $x\in\mathbb{X}$, assuming that the autoencoder~\cref{eq:autoencoder} and reduced order dynamics~\cref{eq:reduced-order-IVP} found simultaneously.
\end{proof}

In practice the attractor, $\mathbb{X}$, is not fully known, and in the small-data problem, the known points might be a biased representation thereof. Therefore, the encoder must practically accept all points in $\mathbb{R}^n$ as its input. The set defined by the image of $\mathbb{R}^n$ under the encoder, $\theta(\mathbb{R}^n) = \widehat{\mathbb{U}}$ we call the \textit{total reduced order space}, as it fully covers all possible inputs to the encoder.

\begin{definition}\label{def:weak-preservation-condition}
If the reduced order dynamics~\cref{eq:reduced-order-IVP} evolve in a way such that the final state is always in this total reduced order space $u_{i+1}\in\widehat{\mathbb{U}}$, we term this the \textit{weak preservation condition}. 
\end{definition}

We now motivate our subsequent discussion with the following trivial result.
\begin{theorem}\label{eq:compactness-preservation}
A shallow neural network with one hidden layer,
\begin{equation}
    \nu(x) = \*A_2\,\sigma\!\left(\*A_1 x + \*b_1\right) + \*b_2,
\end{equation}
with continuous activation function $\sigma$, maps compact sets to compact sets.
\end{theorem}
\begin{proof}
    Compactness is preserved by continuous functions, thus the composition of an affine, continuous, and another  affine function is itself continuous.
\end{proof}

We can attempt to learn $\widehat{\mathbb{U}}$ for a given encoder, and force that the dynamics lie on this set. This is inadequate for one simple reason: there is no guarantee that $\widehat{\mathbb{U}}$ is compact, thus being a poorly positioned solution to the problem of modeling chaotic dynamics.

In this work we aim to solve this problem by turning it on its head: instead of learning $\widehat{\mathbb{U}}$ we instead \textit{synthetically assign} $\widehat{\mathbb{U}}$ to be some compact lower-dimensional manifold embedded in $\mathbb{R}^r$.
We can then restrict both the encoder~\cref{eq:autoencoder} and dynamics~\cref{eq:reduced-order-IVP} to map onto and evolve on this manifold.
This type of restriction can be performed through a simple algebraic constraint. 
We first take the autoencoder reduced order model~\cref{eq:autoencoder-ROM}, and augment it,
\begin{equation}\label{eq:SyCo-AE-ROM}
\framebox{$
    \begin{gathered}
        u = \theta(x),\quad \widetilde{x} = \phi(u),\\
        \frac{\mathrm{d} u}{\mathrm{d} t} = f(u),\quad 0 = g(u),
    \end{gathered}
    $}
\end{equation}
where the addition of the synthetic constraint function $g: \mathbb{R}^r \to \mathbb{R}^s$ with $s \leq r$ has an effect on both the encoder $\theta$ and on the dynamics $f$. The resulting dynamics define a constrained ordinary differential equation~\cite{ascher1998computer}. This synthetically constrained autoencoder (SyCo-AE) is described visually by~\cref{fig:SyCo-AE-Visual}.

The hope is that synthetically constraining the dynamics to a known low dimensional set would act as a source of additional knowledge.
This knowledge should guide the neural networks to be a good approximation of the underlying dynamics with significantly fewer data. We can also alternatively think about implicitly defining a constraint in the full space. We conjecture:
\begin{conjecture}\label{conj:constraint}
   We approximate a complex, not-yet-known, constraint on the full space, $0 = G(x) \approx G(\phi(u))$ for all reduced order states $u$ that satisfy our synthetic constraint $0 = g(u)$.
\end{conjecture}
In other words, by introducing a \textit{known and simple} constraint on the reduced order dynamics, we learn some \textit{unknown and complex} constraint of the full order dynamics. Note that the range of $G$ and $g$ might not have the same dimension, meaning that the optimal choice of $g$ should somehow be informed by the dimension of the range of the unknown $G$.

One simple way to solve enforce the constraint $g$ in~\cref{eq:SyCo-AE-ROM} for both the encoded and dynamics is through projection. Given the constraint $g$, we can pose the projection operator as minimizing a cost,
\begin{equation}\label{eq:projection}
    \Pi(u) = \argmin_{\widehat{u}} \left\lVert\widehat{u} - u\right\rVert_2^2,\qquad\text{such that}\quad 0 = g(\widehat{u}),
\end{equation}
which for certain choices of $g$ can have a closed form solution. 
Given an arbitrary encoder $\widehat{\theta}$, the projected autoencoder is given by,
\begin{equation}\label{eq:projected-encoder}
    \theta(x) = \Pi\left(\widehat{\theta}(x)\right).
\end{equation}

In this work for the solution of all continuous dynamical systems, in the interest of speed, we utilize the projected adaptive time-step explicit trapezoidal rule,
\begin{equation}\label{eq:adaptive-trapezoidal-rule}
\begin{gathered}
k_1 = h_j\, f(u_j),\quad k_2 = h_j\, f(u_j + k_1),\\
    u_{j+1} = \Pi\left[u_j + \frac{1}{2}(k_1 + k_2)\right],\quad
    \widetilde{u}_{j+1} = \Pi\left[u_j + k1\right],
\end{gathered}
\end{equation}
where 
$h_j$ is the internal time-step of the solver, 
$u_{j+1}$ is the next state at time $t_j + h$, 
$\widetilde{u}_{j+1}$ is the solution of the `embedded' method used to adaptively determine $h_j$ through the canonical method described in~\cite{hairer1991solving},
and the function $\Pi$ from~\cref{eq:projection} projects the solutions onto the manifold defined by $g$ in~\cref{eq:SyCo-AE-ROM} or is the identity function for all other methods. The ideas for this method are a simple combination of ideas found in~\cite{hairer1991solving,Hairer2,hairer2006geometric,ascher1998computer}.


We now turn our attention to the choice of $\widehat{\mathbb{U}}$, implicitly defined by the constraint $g$ in~\cref{eq:SyCo-AE-ROM}. We list a few properties that are nice to have:
(i) $\widehat{\mathbb{U}}$ is a compact (and connected) subset of $\mathbb{R}^r$,
(ii) its topological dimension~\cite{heinonen2001lectures} is as close to $r$ as possible, and
(iii) the projection $\Pi$ in~\cref{eq:projection} has a closed form solution.
One simple candidate for such a set is the sphere $S^{r-1}$, with corresponding constraint and projection of,
\begin{equation}\label{eq:sphere}
    g(u) = \norm{u}_2 - 1,\quad \Pi(u) = \frac{u}{\norm{u}_2},
\end{equation}
that has topological dimension of $r-1$. We make use of the spherical constraint~\cref{eq:sphere} for the remainder of this work, and table the discussion of other possible constraints for future work.

We organize the data into trajectories, with each trajectory, $\{X_k\}_{k=0}^K$, consisting of data at times $\{t_k\}_{k=0}^K$. 
In order to effectively learn the dynamics we explicitly \textit{roll out}~\cite{uy2022operator}~\cref{eq:SyCo-AE-ROM}, by explicitly solving the constrained initial value problem during training.
The degree to which roll out is performed is determined by the length of the trajectory $K$, thus we term $K$ the roll out parameter.
We write $\mathcal{M}_{t_k\to t_{k+1}}(u_k)$ for the solution of the constrained initial value problem~\cref{eq:SyCo-AE-ROM} with an algorithm such as~\cref{eq:adaptive-trapezoidal-rule}.
Note that each trajectory can be of variable length and with variable inter-trajectory time-step, however this idea is not explored in this work.

The cost function, therefore has to take into account four distinct parts: 
(i) the autoencoder error~\cref{eq:autoencoder}, 
(ii) the right-invertability condition~\cref{eq:right-invertibility},
(iii) the error of the dynamics~\cref{eq:SyCo-AE-ROM} in the full space (through the decoder),
(iv) the error of the dynamics~\cref{eq:SyCo-AE-ROM} in the reduced order space, required by~\cref{prop:strong-condition}.
The resulting cost function,
\begin{equation}\label{eq:partial-cost-function}
\begin{aligned}
    \left.\ell_i(\{t_k\},\{X_k\})\right|_W = 
    &\phantom{+\,\,\omega} \sum_{k=0}^K\underbrace{\norm{X_k - \phi(\theta(X_k))}^2_2}_{\text{Autoencoder error}}\,\, +\,\,\omega\sum_{k=0}^K\underbrace{\norm{\theta(X_k) - \theta(\phi(\theta(X_k)))}^2_2}_{\text{Right-inverse error}}\\
    &+\phantom{\omega}\sum_{k=1}^K e^{-2\lambda(t_k - t_0)}\underbrace{\norm{X_k - \phi(\mathcal{M}_{t_0\to t_k}(\theta(X_0)))}^2_2}_{\text{Full space dynamics error}}\\
    &+\upsilon\sum_{k=1}^K e^{-2\lambda(t_k - t_0)}\underbrace{\norm{\theta(X_k) - \mathcal{M}_{t_0\to t_k}(\theta(X_0))}^2_2}_{\text{Reduced space dynamics error}}
\end{aligned},
\end{equation}
is implicitly defined in terms of the neural network weights $W$. The constant $\omega$ manipulated the the right inverse error enabling the autoencoder to be weakly right-invertible~\cref{eq:right-invertibility}. A large choice, $\omega= 10^2$, has been shown~\cite{popov2023multifidelity} to significantly improve the performance of autoencoder-based reduced order models. This value of $\omega$ is used in this work.
For chaotic systems, the dominant error of the system grows exponentially on average, thus the errors at different times forward into the model have to be scaled accordingly. 
This accumulation of the error is regulated by $\lambda$ which can be thought of as the approximation of the largest Lyapunov exponent (LLE)~\cite{parker2012practical} for the given dynamics. 
For a system that is not chaotic, and which can be exactly reconstructed by a reduced order model of dimension $r$, the LLE can be set to $\lambda = 0$, meaning that we assume there is no accumulation of error forward in time, which might not strictly be the case.
For chaotic systems, $\lambda$ is either known, or can be tuned as a hyperparameter.
Finally, $\upsilon$ is a parameter that scales the loss in the reduced order space with respect to the autoencoder loss.

Taking a collection of $I$ trajectories, $\{\{X_k\}_i\}_{i=1}^I$, the full cost function is the expected value with respect to all the trajectories,
\begin{equation}\label{eq:full-cost-function}
    \!L(W) = \mathbb{E}\left[\left.\ell_i(\{t_k\}_i, \{X_k\}_i)\right\rvert_W\right].
\end{equation}


\section{Numerical Experiments}
\label{sec:numerical-experiments}

\pgfplotsset{clean/.style={axis lines*=left,
        axis on top=true,
        axis x line shift=0.0em,
        axis y line shift=0.75em,
        every tick/.style={black, thick},
        axis line style = ultra thick,
        tick align=outside,
        clip=false,
        major tick length=4pt},
log x ticks with fixed point/.style={
      xticklabel={
        \pgfkeys{/pgf/fpu=true}
        \pgfmathparse{exp(\tick)}%
        \pgfmathprintnumber[fixed relative, precision=3]{\pgfmathresult}
        \pgfkeys{/pgf/fpu=false}
      }
  },
  log y ticks with fixed point/.style={
      yticklabel={
        \pgfkeys{/pgf/fpu=true}
        \pgfmathparse{exp(\tick)}%
        \pgfmathprintnumber[fixed relative, precision=3]{\pgfmathresult}
        \pgfkeys{/pgf/fpu=false}
      }
  }}

\begin{figure}
    \centering
    \includegraphics[width=\linewidth]{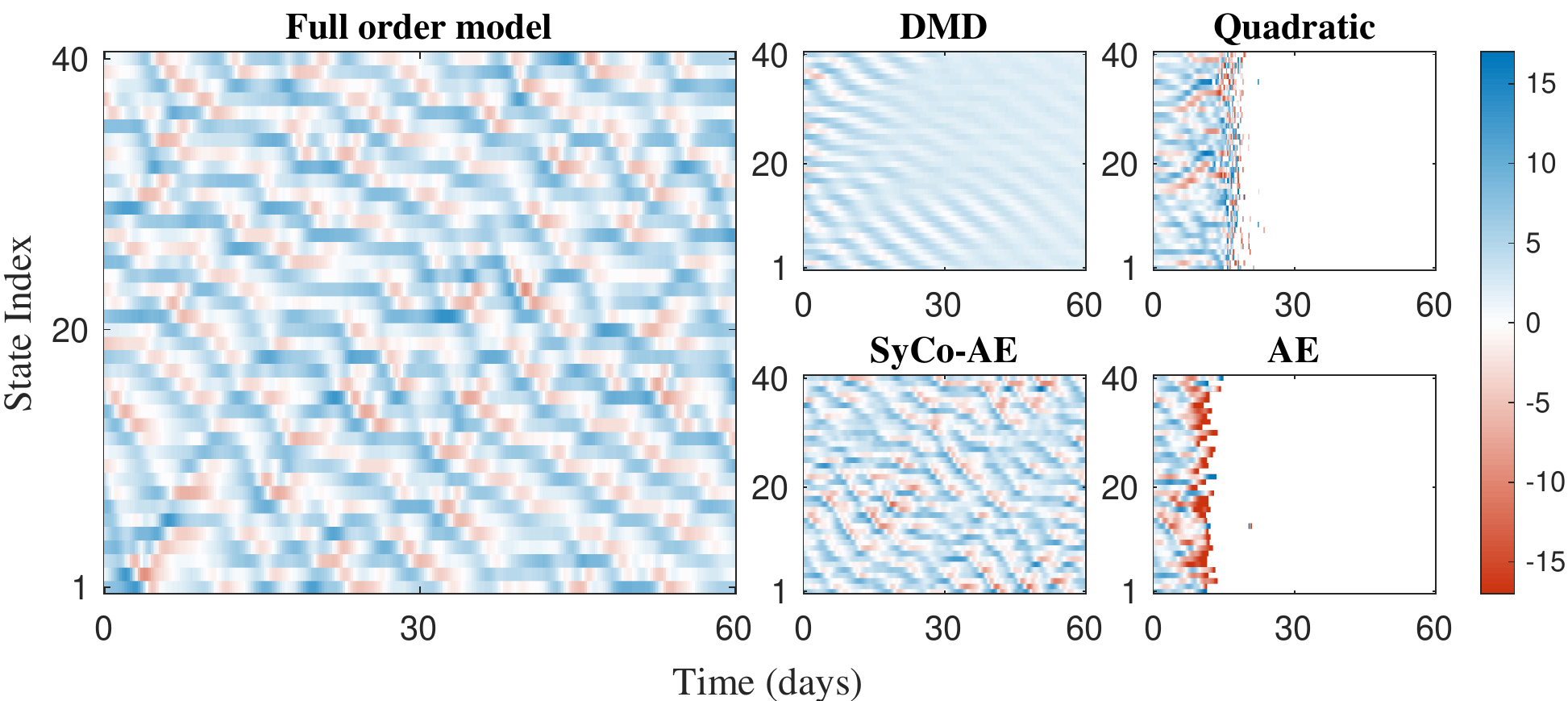}
    \caption{Qualitative representation of the forecasting accuracy of the various ROM methods as compared to the full order Lorenz '96 model. The same initial condition is taken for all models and propagated for 60 days. The $x$-axis is the time in days, and the $y$-axis is the state-space index of each variable.}
    \label{fig:L96-flow}
\end{figure}

\begin{figure}[tp]
    \centering
    \subfigure[$N=100$ data points]{
    \begin{tikzpicture}
    \begin{semilogyaxis}[clean,
        cycle list name=tol,
        xtick={1,2,...,10},
        table/col sep=comma,
        xmin = 1,
        xmax = 10,
        ymin = 10,
        ymax = 1e6,
        restrict expr to domain={y}{10:1e6},
        unbounded coords=discard,
        xlabel = {Time (days)},
        ylabel = {Testing KL-divergence},
        every axis plot/.append style={line width=2pt, mark size=3.5pt},
        legend style={at={(0.6,0.85)},anchor=center},
        legend cell align={left},
        legend columns={2},
        legend to name={leg:forecast},
        scale only axis, width=0.3\textwidth]
    \addplot table [x=daysfull, y=DMDro1, col sep=comma] {dataN/DMD_N100.csv};
    \addlegendentry{DMD };
    \addplot table [x=daysfull, y=QUADro1, col sep=comma] {dataN/QUAD_N100.csv};
    \addlegendentry{Quadratic };
    \addplot table [x=daysfull, y=SyCoAEro1h2000, col sep=comma] {dataN/SyCoAE_N100.csv};
    \addlegendentry{SyCo-AE };
    \addplot table [x=daysfull, y=AEro1h2000, col sep=comma] {dataN/AE_N100.csv};
    \addlegendentry{AE };
    \end{semilogyaxis}
    \end{tikzpicture}
    }%
    \subfigure[$N=500$ data points]{
    \begin{tikzpicture}
    \begin{semilogyaxis}[clean,
        cycle list name=tol,
        xtick={1,2,...,10},
        table/col sep=comma,
        xmin = 1,
        xmax = 10,
        ymin = 10,
        ymax = 1e6,
        restrict expr to domain={y}{10:1e6},
        unbounded coords=discard,
        xlabel = {Time (days)},
        ylabel = {},
        every axis plot/.append style={line width=2pt, mark size=3.5pt},
        legend style={at={(0.6,0.85)},anchor=center},
        legend cell align={left},
        legend columns={2},
        scale only axis, width=0.3\textwidth]
    \addplot table [x=daysfull, y=DMDro1, col sep=comma] {dataN/DMD_N500.csv};
    \addlegendentry{DMD };
    \addplot table [x=daysfull, y=QUADro1, col sep=comma] {dataN/QUAD_N500.csv};
    \addlegendentry{Quadratic };
    \addplot table [x=daysfull, y=SyCoAEro1h2000, col sep=comma] {dataN/SyCoAE_N500.csv};
    \addlegendentry{SyCo-AE };
    \addplot table [x=daysfull, y=AEro1h2000, col sep=comma] {dataN/AE_N500.csv};
    \addlegendentry{AE };
    \legend{};
    \end{semilogyaxis}
    \end{tikzpicture}
    }
    \subfigure[$N=1000$ data points]{
    \begin{tikzpicture}
    \begin{semilogyaxis}[clean,
        cycle list name=tol,
        xtick={1,2,...,10},
        table/col sep=comma,
        xmin = 1,
        xmax = 10,
        ymin = 10,
        ymax = 1e6,
        restrict expr to domain={y}{10:1e6},
        unbounded coords=discard,
        xlabel = {Time (days)},
        ylabel = {Testing KL-divergence},
        every axis plot/.append style={line width=2pt, mark size=3.5pt},
        legend style={at={(0.6,0.85)},anchor=center},
        legend cell align={left},
        legend columns={2},
        scale only axis, width=0.3\textwidth]
    \addplot table [x=daysfull, y=DMDro1, col sep=comma] {dataN/DMD_N1000.csv};
    \addlegendentry{DMD };
    \addplot table [x=daysfull, y=QUADro1, col sep=comma] {dataN/QUAD_N1000.csv};
    \addlegendentry{Quadratic };
    \addplot table [x=daysfull, y=SyCoAEro1h2000, col sep=comma] {dataN/SyCoAE_N1000.csv};
    \addlegendentry{SyCo-AE };
    \addplot table [x=daysfull, y=AEro1h2000, col sep=comma] {dataN/AE_N1000.csv};
    \addlegendentry{AE };
    \legend{};
    \end{semilogyaxis}
    \end{tikzpicture}
    }%
    \subfigure[$N=5000$ data points]{
    \begin{tikzpicture}
    \begin{semilogyaxis}[clean,
        cycle list name=tol,
        xtick={1,2,...,10},
        table/col sep=comma,
        xmin = 1,
        xmax = 10,
        ymin = 10,
        ymax = 1e6,
        restrict expr to domain={y}{10:1e6},
        unbounded coords=discard,
        xlabel = {Time (days)},
        ylabel = {},
        every axis plot/.append style={line width=2pt, mark size=3.5pt},
        legend style={at={(0.6,0.85)},anchor=center},
        legend cell align={left},
        legend columns={2},
        scale only axis, width=0.3\textwidth]
    \addplot table [x=daysfull, y=DMDro1, col sep=comma] {dataN/DMD_N5000.csv};
    \addlegendentry{DMD };
    \addplot table [x=daysfull, y=QUADro1, col sep=comma] {dataN/QUAD_N5000.csv};
    \addlegendentry{Quadratic };
    \addplot table [x=daysfull, y=SyCoAEro1h2000, col sep=comma] {dataN/SyCoAE_N5000.csv};
    \addlegendentry{SyCo-AE };
    \addplot table [x=daysfull, y=AEro1h2000, col sep=comma] {dataN/AE_N5000.csv};
    \addlegendentry{AE };
    \legend{};
    \end{semilogyaxis}
    \end{tikzpicture}
    }
    \ref*{leg:forecast}
    \caption{Medium-range forecasting experiment for the Lorenz '96 equations and the corresponding ROMs of dimension $r=28$ for various amounts of training data.
    For each subfigure, the $x$-axis represents the forecasting time in days, ranging from one to 10, and the $y$-axis represents the KL-divergence~\cref{eq:KL-devergence} of the forecasted data points from the truth. Subfigure (a) has results for models trained on $N=100$ data points, (b) on $N=500$ data, (c) on $N=1000$ data, and (d) on $N=5000$ data.}
    \label{fig:lorenz96-forecast}  
\end{figure}
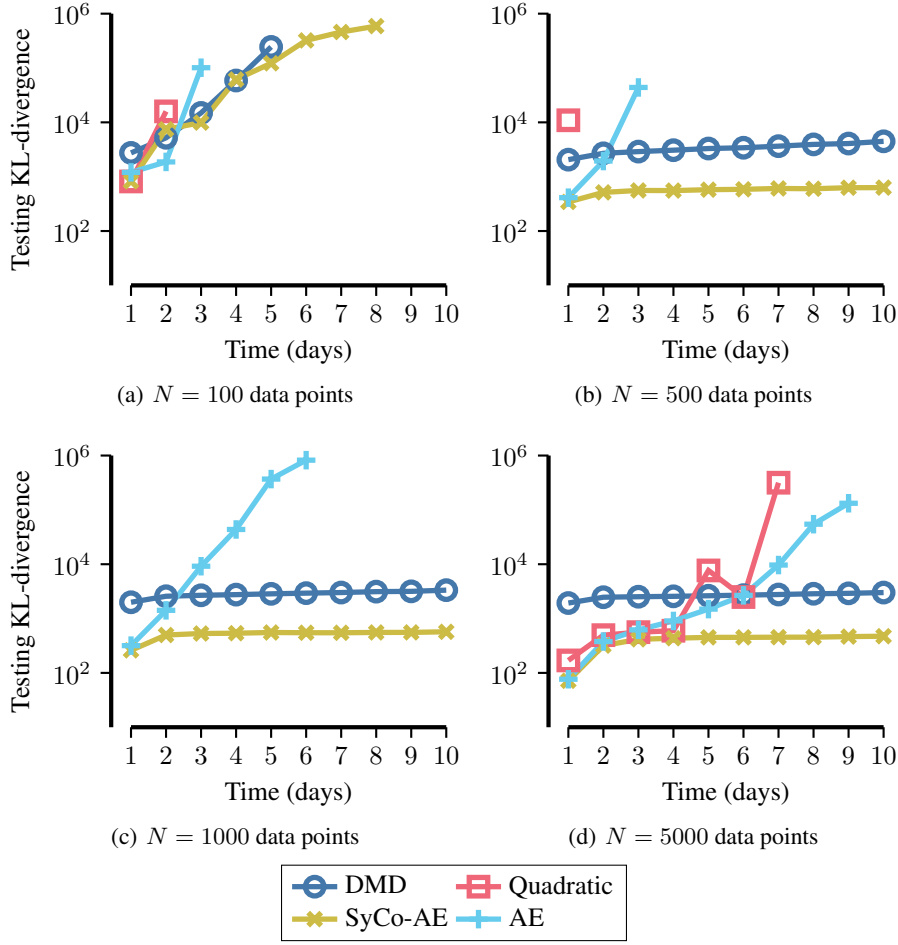

For our numerical test case we take the Lorenz '96 equations~\cite{lorenz1996predictability},
\begin{equation}
    \frac{\mathrm{d} x_j}{\mathrm{d} t} = -x_{j-1}\left(x_{j-2} - x_{j+1}\right) - x_j + F, \quad j\in[1,\dots, 40],
\end{equation}
with cyclic boundary conditions in $j$, and external force $F = 8$.
The implementation of the Lorenz '96 equations used in this work is taken from the ODE Test Problems~\cite{roberts2019ode} package.

The Lorenz '96 equations are a foundational test-bed for algorithms related to numerical weather prediction and data assimilation~\cite{reich2015probabilistic,asch2016data}.
They aim to represent a quantity of interest such as pressure along a slice of fixed latitude on the earth.
We aim to test the reduced order models' ability to perform both medium-range (five to 10 days) and long range ($>10$ days) forecasting~\cite{kalnay2003atmospheric}.

While the Lorenz '96 equations are widely used, it is extremely important to note that it has recently been shown that the equations do not represent a discretization of some partial differential equation model~\cite{van2018dynamics}.

The Lorenz '96 model has a known Kaplan-Yorke dimension of $27.1$ and largest Lyapunov exponent of approximately $\lambda \approx 1.6852$, which were independently computed using a method found in~\cite{dieci2011numerical}. 
The value of $r$ is generally chosen based on computational limitations, and is thus not treated as a hyperparameter.
For these reasons, we focus on the interesting case of $r=28$ for the size of the reduced order model, though preliminary experiments were performed on other values of $r$. 

We compare the  SyCo-AE method~\cref{eq:SyCo-AE-ROM} proposed in this work with the standard autoencoder reduced order model~\cref{eq:autoencoder-ROM}, with the linear DMD method~\cref{eq:DMD-ROM}, and with the non-linear Quadratic manifolds~\cref{eq:quadratic-rom} method.
Both the autoencoder-based methods are trained to minimize the full cost function~\cref{eq:full-cost-function} (made up of~\cref{eq:partial-cost-function}) using the ADAM~\cite{kingma2014adam} method with parameters $\beta_1 = 0.9$ and $\beta_2=0.95$, for 1000 epochs, with no minibatching, and a cyclic learning rate scheduler~\cite{smith2017cyclical} in order to eliminate learning rate as a hyperparameter.

The data collected are trajectories of points collected six hours apart, with each trajectory starting 30 days from the start of the previous, with 0.2 time units corresponding to one day in the model. Not all data points and trajectories are utilized in order to test the methods in the small-data regime.

We train our models on various amounts of data points $N$. We take $N=100$, $N=500$, $N=1000$ and $N=5000$ data points, attempting to range from extremely small, to reasonably large. 
We take two different values of the roll out factor, $K$ in~\cref{eq:partial-cost-function}, namely  $K = 1$ and $K = 9$. 
Note that $N$ does not represent the amount of trajectories that we take, but instead represents the total amount of data points.
For example, for $N=100$, a roll out of $K = 1$ would correspond to 50 trajectories consisting of two data points each, while a roll out of $K = 9$ corresponds to 10 trajectories of 10 points each.
In our testing the roll-out factor $K$ did not play a significant role, thus all results are presented with the roll out factor $K=1$, though data for $K=9$ are available.

For both the autoencoder-based methods, we use a fully connected shallow neural network architecture as described in~\cref{eq:compactness-preservation}, for the encoder $\theta$, the decoder $\phi$, and the reduced order dynamics $f$ in~\cref{eq:autoencoder-ROM} and~\cref{eq:SyCo-AE-ROM}.
For the hidden layer size we take $H=2000$, though models with $H=500$ were also constructed with slightly worse performance.
The encoder for the SyCo-AE method is additionally projected as in~\cref{eq:projected-encoder}. 
As~\cref{eq:compactness-preservation} requires a continuous activation function, we take the GELU function~\cite{hendrycks2016gaussian} as our nonlinearity.

All computation was performed on a MacBook M2 Pro laptop, restricting the number of models that could be trained, thus not all hyperparameters could be optimized for.

For the first experiment we look at the qualitative features of all the methods, as compared to the true full order model for a long range forecast of 60 days. 
We take all models trained on $N=5000$ data point with $K=1$, and $H=2000$ where applicable and forecast a single state outside of the training set.

The results, shown in~\cref{fig:L96-flow}, reveal that the full order model exhibits quasi-periodic behavior, typical of chaotic systems. DMD~\cref{eq:DMD-ROM} exhibits decaying behavior, confirmed by the eigenvalues of $\*\Omega$ all having negative real part. Quadratic manifold~\cref{eq:quadratic-rom} appears to exhibit the correct behavior for approximately 15 days, then appears to catastrophically diverge, which is a known problem in quadratic approximations to dynamical systems~\cite{kaptanoglu2021promoting}. The standard autoencoder approach also appears to have the correct behavior for about five days, but also appears to catastrophically diverge. Only the SyCo-AE method appears to have the same quasi-periodic behavior as the full order model for the full 60 days of the long-range forecast.

In our final experiment we look at the quantitative difference in medium-range forecasting by the reduced order models.
There is evidence to suggest that the Lorenz '96 system is ergodic~\cite{fatkullin2004computational}, roughly meaning that any spatial uncertainty is equivalent to temporal uncertainty.
We make use of the approximate KL divergence~\cite{kullback1951information,schulman2020approximating} of the data distribution $p$ to the true distribution $q$ in the following numerically stable way, 
\begin{equation}\label{eq:KL-devergence}
    \DKL{p}{q} \approx \frac{1}{M}\sum_{x_i\sim p, i =1}^M \frac{1}{2}\left(
    \log q(x) - \log p(x)\right)^2,
\end{equation}
with $M$ representing the number of samples taken, and the distributions approximated by Gaussian mixture models computed through kernel density estimation~\cite{silverman1986density}.

We propagate $M=10000$ states of the full order and reduced order models forward in time from one to 10 days, and calculate the KL-divergence over time for various data sizes.
The results in \cref{fig:lorenz96-forecast} show a clear distinction between the four different methods.

All methods perform poorly for $N=100$ data points, with SyCo-AE~\cref{eq:SyCo-AE-ROM} being the most stable, though not by much.
For $N=500$ and $N=1000$ data points, both DMD~\cref{eq:DMD-ROM} and SyCo-AE~\cref{eq:SyCo-AE-ROM} reach a steady state error, with SyCo-AE having a KL-divergence of an order of magnitude less. 
The standard autoencoder method is able to produce useful two day forecasts, but then diverges.
For $N=5000$ data points the quadratic manifolds method and the standard autoencoder are finally able to produce useful forecasts of about five to six days, but still catastrophically diverge afterwards.

\section{Conclusions and Limitations}

We provided a new framework for reduced order modeling of chaotic systems with neural networks, synthetically constrained autoencoders (SyCo-AE)~\cref{eq:SyCo-AE-ROM} that augments a standard autoencoder-based ROM~\cref{eq:autoencoder-ROM} with a synthetic constraint.
This simple synthetic constraint in the reduced space stands in as a proxy for an unknown constraint of the system in the full space.
Thus, by learning the autoencoder and reduced order dynamics restricted to the simple manifold determined by the synthetic constraint, we implicitly learn the manifold of the full order system.

We test our method on the Lorenz '96 equations. Results show that the SyCo-AE approach can create stable reduced order models capable of medium-range forecasting that are more accurate than methods that do not enforce a compact constraint on the total reduced order space.
Additionally, results show that the amount of data needed to train useful reduced order models is significantly less than other non-linear methods, and seems to be the same as required to train a linear method.

Some of the limitations of this work are as follows.
Alternative choices of the the synthetic constraint $g$ are not explored, and a robust justification of the choice of the sphere is absent.
A more robust exploration of the tunable hyperparameters and neural network architecture is required.
The method in this work was only tested on a simple 40-variable problem, do to computational constraints. 
A larger exploration of more models, including more geophysically realistic models would strengthen the arguments for or against the current work.

Future work would more formally explore the connection between a spherical synthetic constraint $g$ and spherical embedding~\cite{kosinski2013differential}.

\begin{ack}
The first author would like to thank Reid Gomillion and Steven Roberts for their insightful discussion on time integration methods.
This work was sponsored in part by AFOSR (Air Force Office of Scientific Research) under contract number FA9550-22-1-0419.
\end{ack}

\bibliographystyle{plain}
\bibliography{bib/chaos,bib/differentialgeometry,bib/L96,bib/qg, bib/nn,bib/multifidelity,bib/flowmap,bib/filteringgeneral,bib/timeintegration,bib/dimensionalityreduction,bib/control}


\end{document}